\documentclass[10pt,twocolumn,letterpaper]{article}

\usepackage{tikz}
\usepackage{cvpr}
\usepackage{times}
\usepackage{epsfig}
\usepackage{rotating}
\usepackage{graphicx}
\usepackage{pgfplots}
\usepackage{amsmath,amssymb} 
\usepackage{amsthm}
\usepackage{makecell}
\usepackage{algorithm}
\usepackage{algpseudocode}
\usepackage{color}
\usepackage{url}
\usepackage{bigstrut}
\usepackage{multirow}
\usepackage{amsfonts}
\usepackage{makecell}
\usepackage{siunitx}
\usepackage[super]{nth}
\usepackage{subcaption}
\usepackage{booktabs}
\usepackage{array}
\usepackage{caption}

\usepackage[breaklinks=true,bookmarks=false]{hyperref}

\newtheorem{lemma}{Lemma}

\cvprfinalcopy  

\def\cvprPaperID{1916} 

\newlength{\dhatheight}

\newcommand{\norm}[1]{\left\lVert#1\right\rVert}

\begin{document}

\title{Linking Image and Text with 2-Way Nets} 

\author{Aviv Eisenschtat and Lior Wolf\\
The Blavatnik School of Computer Science\\
Tel Aviv University\\
{\tt\small eisen.aviv@gmail.com, wolf@cs.tau.ac.il}}

\date{}

\maketitle

\begin{abstract}
Linking two data sources is a basic building block in numerous computer vision problems.  Canonical Correlation Analysis (CCA) achieves this by utilizing a linear optimizer in order to maximize the correlation between the two views.  Recent work makes use of non-linear models, including deep learning techniques, that optimize the CCA loss in some feature space. In this paper, we introduce a novel, bi-directional neural network architecture for the task of matching vectors from two data sources. Our approach employs two tied neural network channels that project the two views into a common, maximally correlated space using the Euclidean loss. We show a direct link between the correlation-based loss and Euclidean loss, enabling the use of Euclidean loss for correlation maximization. To overcome common Euclidean regression optimization problems, we modify well-known techniques to our problem, including batch normalization and dropout. We show state of the art results on a number of computer vision matching tasks including MNIST image matching and sentence-image matching on the Flickr8k, Flickr30k and COCO datasets.
\end{abstract}

\section{Introduction}

Computer vision emerged from its roots in image processing when researchers began to seek an understanding of the scene behind the image. Linking visual data $X$ with an external data source $Y$ is, therefore, the defining task of computer vision. When applying machine learning tools to solve such tasks, we often consider the outside source $Y$ to be univariate, e.g., in image classification. A more general scenario is the one in which $Y$ is also multidimensional. Examples of such view to view linking include matching between video and concurrent audio, matching an image with its textual description, matching images from two fixed views, etc.

The classical method of matching vectors between two different domains is  Canonical Correlation Analysis (CCA). The algorithm has been generalized in many ways: regularization was added ~\cite{martin1979multivariate}, kernels were introduced~\cite{kcca1,kcca2,kcca3}, versions for more than two sources were developed~\cite{Tenenhaus2011} etc. Recently, with the advent of deep learning methods, deep versions were created and showed promise. 

The current deep CCA methods optimize the CCA loss on top of a deep neural network architecture. In this work, an alternative is presented in which a network is built to map one source $X$ to another source $Y$ and back. This architecture, which bears similarities to the encoder-decoder framework~\cite{hintonnature}, employs the Euclidean loss. 

The Euclidean loss is hard to optimize for, when compared to classification losses such as the cross entropy loss. We, therefore, introduce a number of contributions that are critical to the success of our methods. These include: (i) a mid-way loss term that helps support the training of the hidden layers; (ii) a decorrelation regularization term that links the problem back to CCA; (iii) modified batch normalization layers; (iv) a regularization of the scale parameter that ensures that the variance does not diminish from one layer to the next; (v) a tied dropout method; and (vi) a method for dealing with high-dimensional data.

Taken together, we are able to present a general and robust method. In an extensive set of experiments, we present clear advantages over both the classical and recent methods.

\section{Previous work}
Canonical Correlation Analysis (CCA)~\cite{cca} is a statistical method for computing a linear projection for two views into a common space which maximizes their correlation. CCA plays a crucial role in many computer vision applications including multiview analysis~\cite{conf/cvpr/SharmaKDJ12}, multimodal human behavior analysis~\cite{song_multimodal_2012}, action recognition~\cite{4547427}, and linking text with images~\cite{Klein_2015_CVPR}. There are a large number of CCA variants including: regularized CCA~\cite{regularized_cca}, Nonparametric canonical correlation analysis (NCCA) ~\cite{DBLP:journals/corr/MichaeliWL15}, and Kernel canonical correlation analysis (KCCA)~\cite{kcca1,kcca2,kcca3}, a method for producing non-linear, non-parametric projections using the kernel trick. Recently, randomized non-linear component analysis (RCCA)~\cite{rcca} emerged as a low-rank approximation of KCCA. 

While CCA is restricted to linear projections, KCCA is restricted to a fixed kernel. Both methods do not scale well with the size of the dataset and the size of the representations. A number of methods ~\cite{deepcca,dccae,chandar2016correlational,mae} based on Deep Learning were recently proposed that aim to overcome these drawbacks. Deep canonical correlation analysis~\cite{deepcca} processes the pairs of inputs through two network pipelines and compares the results of each pipeline via the CCA loss. 

~\cite{dcca_text} and ~\cite{DBLP:journals/corr/WangLL15} extend ~\cite{deepcca} to the task of images and text matching. The first employs the same model and training process of ~\cite{deepcca} while the latter employs a different training scheme on the same architecture. Unlike ~\cite{dcca_text} and ~\cite{DBLP:journals/corr/WangLL15} we present a novel deep model for matching images and text.

Other deep CCA methods, including ours, are inspired by a family of encoding/decoding unsupervised generative models~\cite{ae,ae1,sae,denois,sdenois} that aim to capture a meaningful representation of input $x$ by applying a non-linear encoding function $E(x)$, decoding the encoded signal using a non-linear decoding function $D(x)$ and minimizing the squared L2 distance between the original input and the decoded output. Some of the auto-encoder based algorithms incorporate a noise on the input~\cite{denois,sdenois} or enforce a desired property using a regularization term~\cite{sae}. 

Correlation Networks (CorrNet)~\cite{chandar2016correlational} and Deep canonically correlated autoencoders (DCCAE)~\cite{dccae} expand the auto-encoder scheme by considering two input views and two output views. The encoding is shared between the two views (CorrNet) or the differences in the encodings are minimized (DCCAE). In both cases, it serves as a common bottleneck. Our model goes from one view to the other (in both directions) and not from each view to a reconstructed view. 

The CCA loss is used by both CorrNet and DCCAE. The latter contribution explicitly states that the L2 loss is inferior to the CCA loss term~\cite{dccae}. Our network, however, uses L2 successfully. This reinforces the need to apply the methods we propose in this work in order to enable effective training based on the L2 loss. For this end, we introduce innovative techniques based on common practices in deep learning, adapted to the problem at hand. These techniques include: dropout, batch normalization, and leaky ReLUs. While the latter is applied as is, the former two need to be carefully modified for our networks.

Dropout~\cite{dropout} is a regularization method developed to reduce over-fitting in deep neural networks by zeroing a group of neurons at each training iteration. This stochastic elimination reduces the co-adaptation between neurons in the same layer and simulates the training of an ensemble of networks with shared weights. 

Batch Normalization~\cite{bn} is used as a stabilizing   
mechanism for training a neural network by scaling the output of a hidden layer to zero norm and unit variance. This scaling lowers the change of distribution between neurons throughout the network and helps to speed up the training process. 

Rectified Linear Unit (ReLU)~\cite{relu} is a non-linear activation function that does not suffer from the saturation phenomenon, which the classical sigmoids suffer from. Conventional ReLU zero negative activations, and as a result, no gradient is produced for many of the neurons. A few variants of ReLU were, therefore, proposed~\cite{relu2,relu3} that reduce the effect of negative activations, but do not zero them completely. Similar to~\cite{relu2} and unlike~\cite{relu3}, we do not train the leakiness parameter and instead set it to a constant value.

As one of our contributions, we add a regularization term that removes the pairwise covariances of the learned features. A similar term was recently reported in work~\cite{DBLP:journals/corr/CogswellAGZB15} as part of a classification system (unrelated to modeling correlations between vectors). We adapt their terminology when describing our bi-directional term. 

\begin{figure}[t]
\centering
\includegraphics[width=\columnwidth, scale=0.5]{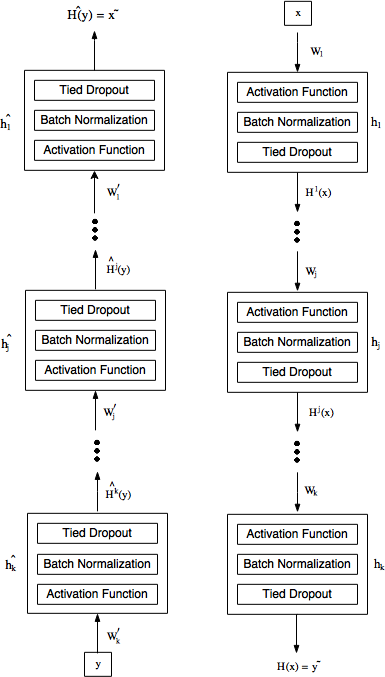}
\caption{The 2-way network model. Each channel transforms one view into the other. A middle representation is extracted for correlation maximization}
\label{fig:arch}
\end{figure}

\section{The Network Model}
\label{sec:arch}
This section contains a detailed description of our proposed model, which we term the 2-way net%
\footnote{Code can be found at ~\url{https://github.com/aviveise/2WayNet}}
. The model utilizes the L2 loss in order to create a bi-directional mapping between two vector spaces. The absence of a correlation based loss (such as in DeepCCA~\cite{deepcca} and CorrNet~\cite{chandar2016correlational}) makes this model simpler. Like other regression problems, there are inherent challenges in obtaining meaningful solutions~\cite{neural-networks-2}. These challenges are further amplified by the multivariate and layered structure of the performed regression. We, therefore, modify the problem in various ways, each contributing to the overall success.

\subsection{Basic Architecture}
Our proposed architecture is illustrated in Fig.~\ref{fig:arch}. It contains two reconstruction channels. Both channels contains $k$ hidden layers $\{h_1, h_2, ..., h_k\}$ and $\{\hat{h_1}, \hat{h_2}, ..., \hat{h_k}\}$. Lets define $H_i(x)$ and $\hat{H_i(y)}$ as the output of each channel at layer $i$ given network inputs $x$ and $y$ respectively, the model is optimized to minimize the Eucledean loss between both $\hat{H_i(y)}$ and $x$, and $H_i(x)$ and y. The two channels share weights and dropout function as explained in ~\ref{sec:drop}

The activations of each hidden layer are computed by a function $h(x)=\Phi\left(Wx+b_2\right)$ from $\mathbb{R}^{d_1}$ to $\mathbb{R}^{d_2}$, where $W\in{\mathbb{R}^{d_2 \times d_1}}$ is the weight matrix, $b_2\in \mathbb{R}^{d_2}$ is the bias vector and $\Phi$ is a non-linear function, which in our model is a leaky rectified linear unit~\cite{relu2}. The tied layer is given as $\hat{h}(y)=\Phi\left(W^Ty+b_1\right)$, and employs the transpose of the matrix $W$ and an untied bias term $b_1 \in \mathbb{R}^{d_1}$. 

Given a pair of views $x \in \mathbb{R}^{d_x}$ and $y \in \mathbb{R}^{d_y}$, two reconstructions are created: $\tilde{x} \in \mathbb{R}^{d_x}$ and $\tilde{y} \in \mathbb{R}^{d_y}$ by employing the two networks $H=h_{1}\circ h_{2}\circ...\circ h_{k}$ and $\hat{H}=\hat{h}_{k}\circ\hat{h}_{x_{k-1}}\circ...\circ\hat{h}_{1}$, as $\tilde{x}=\hat{H}(y)$ and $\tilde{y}=H(x)$. 

Loss is measured between $x$ and $\tilde{x}$ and $y$ and $\tilde{y}$. Moreover, the Euclidean distance is also minimized directly on the desired representations. In order to do so, we select a mid-network position $j=\lceil k/2 \rceil$. We then add a loss term by considering the two networks: $H^j=h_{1}\circ h_{2}\circ...\circ h_{j}$, and $\hat H^j=\hat{h}_{k}\circ\hat{h}_{x_{k-1}}\circ...\circ\hat{h}_{j+1}$. A loss term is then added that compares $H^j(x)$ and $\hat{H}^j(y)$.

The overall loss (sans regularization terms) is given by the three terms $L_x=\|x-\tilde{x}\|^2$, $L_y=\|y-\tilde{y}\|^2$, and $L_h=\|H^j(x)-\hat{H}^j(y)\|^2$. Note that minimizing Euclidean distances differs from maximizing the pairwise correlations as is done in CCA and its variants DeepCCA~\cite{deepcca} and RCCA~\cite{rcca}. 

In our experiments, in order to compare with previous work, we use the correlation as the success metric. As the Lemma below shows, there is a connection between the correlation of two vectors and their Euclidean distance, this connection also depends on the variance of the vectors.  

\begin{lemma}\label{lemma:Euclidean}

Let $x\in\mathbb{R}^{n}$ and $y\in\mathbb{R}^{n}$ denote two paired lists of $n$ matching samples from two random variables with zero mean and $\sigma_{x}^2$ and $\sigma_{y}^2$ variances. Then, the correlation between the two $n$ dimensional samples $x$ and $y$ equals $\frac{\sigma_{x}}{2\sigma_{y}}+\frac{\sigma_{y}}{2\sigma_{x}}-\frac{\norm{x-y}^2}{2n\sigma_{x}\sigma_{y}}$.

\end{lemma}
\begin{proof}
Given two n-dimensional vectors $x$ and $y$ we consider the squared Euclidean distance 
\begin{equation} \label{eq:distance}
\norm{x-y}^2=\sum_{j=1}^{n}(x_{j}^2) + \sum_{j=1}^{n}(y_{j}^2) -2\sum_{j=1}^{n}(x_{j}y_{j})\nonumber
\end{equation}
Thus: 
\begin{equation}
\sum_{j=1}^{n}(x_{j}y_{j})=\frac{n\sigma_{x}^2}{2}+\frac{n\sigma_{y}^2}{2}-\frac{\norm{x-y}^2}{2}
\end{equation}
For zero mean variables, the correlation between $x$ and $y$ is given by $c=\frac{1}{n}\frac{\sum_{j=1}^{n}(x_{j}y_{j})}{\sigma_{x}\sigma_{y}}$. 
Combining with~\ref{eq:distance} results in what had to be proven.
\end{proof}

Given a batch of samples from views $x$ and $y$, we measure the correlation between the outputs of two matching layers, $\{h_j(x_1), ..., h_j(x_n)\}$ and $\{\hat{h}_j(y_i), ..., \hat{h}_j(y_n)\}$ as the sum of correlations between the activations of each matching neuron. The Lemma below extends Lemma \ref{lemma:Euclidean} and shows that the sum of correlations which we aim to maximize is bounded by a function of the Euclidean loss between the two representations.

\begin{lemma}\label{lemma:sumcorr}
Given two matching hidden layers, $h_j$ and $\hat{h_j}$ with $m$ neurons each. $a_k$ is the activation vector of neuron $k$ from $h_j$ with standard deviation $\sigma_{a_k}$ and $b_k$ is the activation vector of neuron $k$ from $\hat{h}_j$ with standard deviation $\sigma_{b_k}$. Each vector is produced by feeding a batch of samples of size $n$ from views $x$ and $y$ through channels $H$ and $\hat{H}$ respectively. The sum of correlations $C$ is bounded by:
\begin{align}
\sum_{k=1}^{m}C_k &\geq \frac{1}{2}\sum_{k=1}^{m}(\frac{\sigma_{a_k}^2+\sigma_{b_k}^2}{\sigma_{a_k}\sigma_{b_k}}) \nonumber 
\\&- \frac{1}{2n}\sum_{k=1}^{m}\norm{a_k -b_k}^2\sum_{k=1}^{m}\sigma_{a_k}^{-1}\sigma_{b_k}^{-1}\label{eq:corrlossconnection}
\end{align}
\end{lemma}
\begin{proof}
From lemma \ref{lemma:Euclidean}, we get:
\begin{equation}\label{eq:sumcorr}
\sum_{k=1}^{m}C_k = \frac{1}{2}\sum_{k=1}^{m}(\frac{\sigma_{a_k}^2+\sigma_{b_k}^2}{\sigma_{a_k}\sigma_{b_k}}) - \frac{1}{2n}\sum_{k=1}^{m}(\frac{\norm{a_k-b_k}^2}{\sigma_{a_k}\sigma_{b_k}})
\end{equation}
We will define $G_m=\sum_{k=1}^{m}\norm{a_k-b_k}^2$ and $f_k=\sigma_{a_k}^{-1}\sigma_{b_k}^{-1}$.
Using Abel transform:
\begin{eqnarray}
\sum_{k=1}^{m}\frac{\norm{a_k-b_k}^2}{\sigma_{a_k}\sigma_{b_k}} 
&=& f_mG_m-\sum_{k=1}^{m-1}G_kf_{k+1}+\sum_{k=1}^{m-1}G_kf_k \nonumber \\ &\leq& f_mG_m+\sum_{k=1}^{m-1}G_kf_k  \nonumber \\&\leq& f_mG_m+G_m\sum_{k=1}^{m-1}f_k \nonumber= G_m\sum_{k=1}^{m}f_k\\&=& \sum_{k=1}^{m}\norm{a_k-b_k}^2\sum_{k=1}^{m}\sigma_{a_k}^{-1}\sigma_{b_k}^{-1} \label{eq:abel}
\end{eqnarray}
Note that both $\sigma_{a_k}\sigma_{b_k}$ and $\norm{a_k-b_k}^2$ are positive for all $k$ which makes the above inequalities valid.
Inserting \ref{eq:abel} in \ref{eq:sumcorr} results in what had to be proven.
\end{proof}
From the above Lemma, we can conclude that by minimizing the L2 loss together with maximizing the variance of each neuron activation will result in maximization of the sum of correlations. 

Solving this regression problem tends to eliminate the variance of the output representations. To overcome this limitation, we add two instruments. The first is batch normalization layer \cite{bn} (BN) after each hidden layer. The settings of the batch normalization layer differ from the common settings to adapt to this model. Another instrument is regularizing the gamma parameter the batch normalization layer introduces. More details can be found below.

To the loss term, we add regularization terms. The first is weight decay $R_w=\sum\|W\|^2$. A second regularization term is added in order to reduce the cross correlations between the network activations of the same layer. The property we encourage is inherent to CCA-based solutions where decorrelation is enforced. In our network solutions, we add a soft regularization term. During training, we consider the $N$ samples of a single batch $\{(x_i,y_i)\}_{i=1}^N$ and consider the set of mid-network activations $\{(H^j(x_i),\hat H^j(y_i))\}_{i=1}^N$. The decorrelation regularization term is given by: 
\begin{equation} \label{eq:deconv}
\begin{split}
R_{decov}&=\frac{1}{2}\left(\|C_h\|^2_F-\|diag\left(C_h\right)\|_2^2\right)\\& +\frac{1}{2}\left(\|C_{\hat{h}}\|^2_F-\|diag\left(C_{\hat{h}}\right)\|_2^2\right)~,
\end{split}
\end{equation}
where $C_h=\frac{1}{N} \sum_i H^j(x_i)^\top H^j(x_i)$ is the covariance estimator for $H^j(x)$ and $C_{\hat{h}}=\frac{1}{N}\sum_i \hat{H}^j(y_i)^\top\hat{H}^j(y_i)$ is the covariance estimator for $\hat{H}^j(y)$. This regularization term is minimized when the off-diagonal coefficients of both $C_h$ and $C_{\hat{h}}$ are zero.

\subsection{Batch normalization layers}\label{sec:bn}
As shown above, in order to maximize the correlation we need not only to minimize the Euclidean loss but also to increase the variance of each neuron's output. This is done by introducing a batch normalization layer \cite{bn} customized to meet the model's needs. 

Given a vector of activations $a = [a_1,\dots,a_d]$ produced by one of the network's hidden layers for a given batch of inputs, we normalize $a$ to produce $a' = [a'_1,\dots,a'_d]$, where ${a_k'}=\frac{a_k-\mu_k}{\sigma_k}$ and $\mu_k$ and $\sigma^2_k$ are the mean and variance of neuron $k$ on the given batch. This is followed by scaling and shifting by learned parameters to produce ${a_k''}={\gamma }_k{a_k'}+{\beta }_k$. The BN layer mitigates the loss of variance by enforcing unit variance and by removing the influence of the weights of the hidden layer on the output's variance.

BN layers are usually placed before the non-linearity or on the input of the layer as a preprocessing phase as shown in \cite{he2016identity}. This setting poses several problems. First, ReLU lowers the variance of the output which is counterproductive to our goal. Second, applying ReLU after BN has the effect of zeroing every $k$ when $a_k$ is below the mean in a given batch plus the term $\beta_k/\gamma_k$.  Typically, $\beta_k$ is initialized to zero and for a symmetric activation distribution, half of the activations are zeroed. When employing a bi-directional network, the zeroing effect occurs in both directions.

In order to estimate the magnitude of this effect, let us assume that we have a process that at time $i$ outputs two vectors $u_i = H^j(x_i)$ and $v_i = \hat H^j(y_i)$, both in $\mathbb{R}^d$, which are the hidden representation at layer $j$ for a pair of samples $(x_i,y_i)$. Denote by $\rho_k$ the correlation between the activations at neuron $k$.

Let $s_i=\{k | u_i(k)>\mu_k\}$ be the group of indices of the values in $u_i$ that are larger than their population mean. Let $\hat s_i=\{k | v_i(k)>\hat \mu_k\}$ be the equivalent for the vectors $v_i$. We observe the intersection $s_i\cap \hat s_i$, which is the group of active neurons, following a threshold at the mean value on both $u_i$ and $v_i$.

As the Lemma below shows, even if the correlation $\rho_k$ is relatively high, the size of the intersection set $s_i\cap \hat s_i$ is closer to the value $d/4$ obtained for randomly permuted vectors than to the maximal value of $d/2$.

\begin{lemma}\label{lemma:bn}
Assume that $u_i$ and $v_i$ are drawn from a multivariate normal distribution with zero mean and the identity covariance matrix, such that the correlation between $u_i(k)$ and $v_i(k)$ for all $k$ is $\rho_k=\rho$. Then,
$E\left(\left| s_i\cap \hat s_i \right|\right)=d\left[\frac{1}{4}+\frac{{{\mathrm{sin}}^{-1} \rho \ }}{2\pi }\right] $.
\end{lemma}

\begin{proof}
To estimate the size of $c$, let us look at the quadrant probability$\ p$ of $u_i(k)$ and $v_i(k)$ which is given analytically by~\cite{kendall},
\[p=P(u_i(k)>0,v_i(k)>0)=\frac{1}{4}+\frac{{{\mathrm{sin}}^{-1} \rho \ }}{2\pi }\] 
Given that the variables in $u_i(k)$ and $v_i(k)$ are drawn independently, the probability of $P(|c|=t)$ has a binomial distribution with probability $p$, thus the mean of the size of $c$ is equal to
$E(|c|)=dp=d\left[\frac{1}{4}+\frac{{{\mathrm{sin}}^{-1} \rho \ }}{2\pi }\right]$. 
\end{proof}

Even in the case of a correlation as high as 0.6, the intersection will include only about 35\% of the neurons. For neurons $k$ not in this intersection, either both sides $u_i(k)$ and $v_i(k)$ are zero, meaning that no backpropagation occurs, or only one neuron is active, in which case only that side is updated and the update is a simple shrinking effect, since the loss is the magnitude of the activation. 

In order to break this symmetry, we choose to employ the BN after the non-linearity. This allows the network to choose weights that result in mostly positive activations, which remain positive after the ReLU activation units. 

\subsection{Highly leaky ReLU}

Another method to prevent the harmful effects of zeroing is by using leaky ReLU as our non-linear function. Leaky ReLU was first introduced by~\cite{relu2} in order to overcome the difficulties that arise from the elimination of the gradients from neurons with negative activation. In the 2-Way network, this effect is amplified, and we find leaky ReLU units to be extremely important. Formally, a leaky ReLU is defined as:
\[
y_i=
\left\{
	\begin{array}{ll}
		x_i  & \mbox{if } x \geq 0 \\
		ax_i & \mbox{if } x < 0
	\end{array}
\right.
\]
where $a<1$ is the leakiness coefficient and is fixed during both training and testing. In all of our experiments, we use a leakiness coefficient of 0.3. This value was selected on the validation set of the Flickr8k experiment described in Section~\ref{sec:exp} and is used for all experiments.

Using leaky ReLU helps to reduce the effect discussed in Section \ref{sec:bn} but does not replace the need for performing BN after the non-linearity. As Lemma ~\ref{lemma:bn} shows, more than half of the neurons will be multiplied by the leakiness coefficient while their matching neuron will not. This asymmetric scaling adds an artificial distance between the matching neurons, which, in turn, increases the L2 loss and reduces the training efficiency. 

\subsection{Variance injection}
Applying BN on the output of each hidden layer is not enough. The variance can still vanish during training. The problem is that the $\gamma$ factor introduced by each BN layer can be arbitrary and can diminish during training, resulting in low variance. To encourage high variance, we introduce a novel regularization term of the form $R_\gamma = \sum_{j,k} (1/\gamma_{jk})^2$, where $\gamma_{jk}$ is the scaling parameter for neuron $k$ in layer $j$.

This regularization term is enough to force the network to avoid solutions with low variance and to seek more informative output. This is demonstrated experimentally in the ablation study of Section~\ref{sec:exp}.

The compound loss term we employ is of the form:
\[L=L_x+L_y+L_h+\lambda_wR_w+\lambda_{decov}R_{decov} +\lambda_\gamma R_\gamma\]
Where $\lambda_w $,$\lambda_{decov}$, and $\lambda_\gamma$ are the regularization coefficients. While it seems that three regularization tradeoff hyperparameters would make selecting the parameter values difficult, the converse is true: in all of our varied set of experiments $\lambda_{\gamma} = \lambda_w$, and $\lambda_{decov}$ is either set to a very high value of $1/2$ or, for small datasets, to $1/20$ (see Section~\ref{sec:exp}). Moreover, by adding these terms, the network is much less sensitive to the selection of $\lambda_w$ and allows us to learn with a much higher learning rate.

\subsection{Tied dropout}\label{sec:drop}
Dropout~\cite{dropout} is a form of regularization method that simulates the training of multiple networks with shared weights. Dropout zeros neurons by element-wise multiplying the output of a hidden layer consisting of $d$ neurons for a batch of $n$ samples with a random matrix $B$ of size $d \times n$. Each element of $B$ is drawn independently from a Bernoulli distribution with a parameter $p$.

Since dropout eliminates random neurons, it prevents co-adaptation of neurons, which is a desirable property for correlation analysis. However, using dropout, as is, in our proposed model is harmful. This is because the 2-Way network aims to enhance correlations between parallel layers $h^j$ and $\hat h^j$. The elimination of neurons independently in the hidden layers creates an artificial loss, even for a perfect matching. 

Let $p$ be the dropout parameter for layer $j$, assume that the same parameter is applied on both directions. In probability $(1-p)^2$, a pair of matching neurons is active on both sides and learning occurs with the true gradient. In probability $p^2$, the pair of matching neurons is silent on both sides and no learning occurs. In probability $2p(1-p)$, only one neuron is active resulting in a shrinking effect on the other neuron. Here, too, shrinking of activations is can be damaging since it might lead to a state of constant representation.

For a dropout probability of $p=0.5$, half of the gradients would stem from a match which is silent on exactly one side, and the harmful effect is clearly seen in Section~\ref{sec:exp}.

To overcome this problem, we introduce a tied dropout layer, in which the same random matrix $B^j$ is applied to pairs of matching hidden layers: $h^j$ and $\hat h^j$, $j=1..K$. This sharing eliminates the artifacts introduced by the conventional dropout while preserving the benefits of the stochastic process and helps avoid over-fitting. 

Using tied dropout layer changes the distribution of the activations. In order to match the distribution at test time, we incorporate a scaling factor at train time. 

Assume that the activations of a single neuron are zero-centered. As discussed below, most post BN activations are almost exactly centered. In this case, the variance of the neuron activations is simply the sum of the squared activations. During training, only a ratio $1-p$ of the activations contributes to the variance. Therefore, we divide the activations, at train time, by $\sqrt{1-p}$.

\subsection{Training high dimensional inputs}
Some of the experiments shown below contain high dimensional data. High dimensional input directly increases the number of parameters and can cause over-fitting as well as an increase in training time and memory usage. To lower the number of parameters, we introduce a new type of layer we term locally dense layer. Such layer of size $n$ is composed of $m$ different dense layers $\bar{h_1}, ..., \bar{h_m}$ of size $\frac{n}{m}$ each. Input $x$ of size $d_x$ is divided into $m$ different parts of size $\frac{d_x}{m}$ and each part $x_i$ is connected into one of the dense layers $\bar{h_i}$. The outputs of all inner hidden layers are concatenated, thus producing the locally dense layer's output. To the output, we add a regular bias term $b$ of size $n$.
Using this layer reduces the number of parameters by a factor of $m$ comparing to a conventional dense layer. In the experiments below, when dealing with high dimensional input, we use a locally dense layer with two inner dense layers.

\section{Experiments}
\label{sec:exp}

We first present a detailed analysis of the two datasets most commonly used in the literature for examining recent CCA variants: MNIST half matching and X-Ray Microbeam Speech data (XRMB). We then provide additional experiments on the problem of image to sentence matching, showing state of the art results on the Flickr8k, Flickr30k and COCO datasets.

\subsection{Comparison with published results}
We follow the conventional way of evaluating the performance of CCA variants and compute the sum of the correlations of the top $c$ shared (canonical) representation variables found. The datasets used for this comparison are MNIST and XRMB. In both MNIST and XRMB experiments, we set $\lambda_{decov} = \lambda_{W} = \lambda_\gamma = 0.05$. For training, we used stochastic gradient descent with a learning rate of 0.0001 which was halved every 20 epochs. A momentum of 0.9 is used and a tied dropout probability of 0.5. 

\noindent {\bf MNIST half matching} The MNIST handwritten digits dataset~\cite{mnist} contains 60,000 images of handwritten digits for training and 10,000 images for testing. 
Each image is cut vertically into two halves, resulting in 392 features each. The goal is to maximize the correlation of the top $c=50$ canonical variables. The model used is composed of three layers of size 392, 50 and 392 respectively, noted as 392-50-392. The middle layer was taken as the output. 

\noindent {\bf X-Ray Microbeam Speech data} The XRMB~\cite{xrmb} dataset contains simultaneous acoustic and articulatory recordings. The articulatory data is represented as a 112 dimensional vector. The acoustic data are the MFCCs~\cite{logan2000mel} for the same frames, yielding a 273 dimensional vector at each point in time. For benchmarking, 30,000 random samples are used for training, 10,000 for cross-validation and 10,000 for testing. The correlation is measured across the $c=112$ top correlated canonical variables. The same training configuration of the MNIST experiment was used for the XRMB dataset. For XRMB, we tested our model using hidden layer configuration of 560-280-112-680-1365. 

Tab.~\ref{tab:exp} contains correlation comparisons on the MNIST and XRMB datasets of six CCA variants besides our proposed method. As can be seen, our method (``2WayNet'') outperforms all literature methods by a large margin on the XRMB dataset. On the MNIST dataset, in which the literature results are closer to the maximal value of 50, our method is able to regain half of the remaining correlation. 

\begin{table}[h]
  \centering
  \begin{tabular}[c]{|c|c|c|}
	\hline
    Method & MNIST & XRMB \\
    \hline
    \multicolumn{1}{|c|}{Regularized CCA~\cite{regularized_cca}} & 28.0 & 16.9 \\
	\multicolumn{1}{|c|}{DCCA~\cite{deepcca}} & 39.7 & 92.9 \\
    \multicolumn{1}{|c|}{RCCA~\cite{rcca}} & 44.5 & 104.5 \\
    \multicolumn{1}{|c|}{DCCAE~\cite{dccae}} & 25.34 & 41.47 \\
    \multicolumn{1}{|c|}{CorrNet~\cite{chandar2016correlational}} & 48.07 & 95.01\\
    \multicolumn{1}{|c|}{NCCA~\cite{DBLP:journals/corr/MichaeliWL15}} & NA & 107.9\\
    \multicolumn{1}{|c|}{2WayNet} & \textbf{49.15} & \textbf{110.18} \\
    \hline
	\end{tabular}%
    \caption{Comparison between various methods on the XRMB and MNIST datasets. The reported values are the sum of the correlations between the learned representations of the two views. Following the literature, in these benchmarks MNIST employs a 50D shared representation space, and XRMB a 112D one.}
    \label{tab:exp}%
\end{table}%

\subsection{Image annotation and search} 
We next evaluate the proposed model on the sentence-image matching task. In this task, each dataset contains a set of images and five matching sentences per image. For each dataset, we test our model on two tasks, searching an image given a query sentence and matching a sentence given an image. We measure our performance on three datasets, Flickr8k~\cite{flickr8k}, Flickr30k~\cite{flickr30k} and COCO~\cite{coco}, each containing 8,000, 30,000 and 123,000 images respectively. 

Images are presented by the representation layer of the VGG network~\cite{vgg} as vectors of size 4096. Sentences are represented using the published code of~\cite{Klein_2015_CVPR}. Among the available text encodings, we employ the concatenation of the Fisher Vector encoding (GMM) and the Fisher Vector of the HGLMM distribution introduced in~\cite{Klein_2015_CVPR}. Each sentence is thus represented as a 36,000D vector. Going from the image to the much larger sentence representation, we trained networks containing two conventional hidden layers of sizes 2000 and 3000 and an additional locally dense layer of 16000 neurons and $m=2$ for Flickr30k and COCO datasets. For Flickr8k, due to the relatively small dataset, we used a dense layer of 4000 neurons. 
Correlation is used as a similarity measure between images and sentences. To this end we use the middle network representations from each channels, resulting in a representation vector of size 3000.

The Flickr8k dataset is provided with training, validation, and test splits. For Flickr30K and COCO, no splits are given, and we use the same splits used by~\cite{Klein_2015_CVPR}. $\lambda_{deconv}$ is set to a value of $1/2$, which almost eliminated all off-diagonal covariances at the middle layer. The other parameters are set as in the MNIST and XRMB experiments.

Tab.~\ref{tab:flickrs} compare our results to the state-of-the-art methods on the image-sentence matching task. We also report results that we computed for the RCCA method~\cite{rcca}. The open implementations of the various deep CCA methods do not seem to scale well enough for this benchmark. Our proposed method achieves best performance almost across all scores, especially in the image annotation task, where we improved by a large margin for the three datasets, and especially when considering the top result (r@1).

\begin{table*}[t]
\centering
\begin{tabular}{|l|c|c|c|c|c|c|c|c|c|c|c|c|}
\hline
\multirow{3}{*}{Model} & \multicolumn{4}{c|}{Flickr8k} & \multicolumn{4}{c|}{Flickr30k} & \multicolumn{4}{c|}{COCO}\\
\cline{2-13}
& \multicolumn{2}{c|}{Search} & \multicolumn{2}{c|}{Annotate} & \multicolumn{2}{c|}{Search} & \multicolumn{2}{c|}{Annotate} & \multicolumn{2}{c|}{Search} & \multicolumn{2}{c|}{Annotate}\\
\cline{2-13}
& r@1 & r@5 & r@1 & r@5 & r@1 & r@5 & r@1 & r@5 & r@1 & r@5 & r@1 & r@5 \\
\hline
 {NIC~\cite{DBLP:conf/cvpr/VinyalsTBE15}} & 19.0 & NA & 20.0 & NA & 17.0 & NA & 17.0 & NA & NA & NA & NA & NA\\
 {SC-NLM~\cite{DBLP:journals/corr/KirosSZ14}} & 12.5 & 37.0 & 18.0 & 40.9 & 16.8 & 42.0 & 23.0 & 50.7 & NA & NA & NA & NA\\
 {m-RNN~\cite{DBLP:journals/corr/MaoXYWY14}} & 11.5 & 31.0 & 14.5 & 37.2 &  22.8 & 50.7 & 35.4 & 63.8 & 29.0 & 42.2 & 41.0 & 73.0\\
 {m-CNN~\cite{DBLP:conf/iccv/MaLSL15}} & 20.3 & 47.6 & 24.8 & 53.7 &  26.2 & 56.3 & 33.6 & 64.1 & 32.6 & 68.6 & 42.8 & 73.1 \\
 {DCCA~\cite{dcca_text}} & 12.7 & 31.2 & 17.9 & 40.3  & 12.6 & 31.0 & 16.7 & 39.3 & NA & NA & NA & NA\\
 {BRNN~\cite{DBLP:conf/cvpr/KarpathyL15}} & NA & NA & NA & NA & 15.2 & 37.7 & 22.2 & 48.2 & 27.4 &60.2 & 38.4 & 69.9\\
 {RNN-FV~\cite{DBLP:RNNfish}} & 23.2 & \textbf{53.3} & 31.6 & 61.2 & 27.4 & 55.9 & 35.9  & 62.5 & 30.2 & 65.0 & 40.9 & 75.0 \\
 {VQA-A~\cite{DBLP:journals/corr/LinP16}} & 17.2 & 42.8 & 24.3 & 52.2 & 24.9 & 52.6 & 33.9 & 62.5 & 37.0 & 70.9 & 50.5 & \textbf{80.1} \\
 {NLBD~\cite{DBLP:journals/corr/WangLL15}} & NA & NA & NA & NA &29.7 & \textbf{60.1} & 40.3 & \textbf{68.9} & 39.6 & \textbf{75.2}  & 50.1 & 79.7\\
 \hline
 {CCA~\cite{Klein_2015_CVPR}} & 21.3 & 50.1 & 31.0 & 59.3 & 23.5 & 52.8 & 35.0 & 62.1 & 25.1 & 59.8 & 39.4 & 67.9\\
 {RCCA~\cite{rcca}} & 18.7 & 31.1 & 11.7 & 19.2 & 22.7 & 34.2 & 28.3 & 48.2 & NA & NA & NA & NA\\
 {2WayNet} & \textbf{29.3} & 49.7 & \textbf{43.4} & \textbf{63.2} & \textbf{36.0} & 55.6 & \textbf{49.8} & 67.5 & \textbf{39.7} & 63.3 & \textbf{55.8} & 75.2 \\
\hline
\end{tabular}
\caption{The recall rates for the Flickr8k, Flickr30k and COCO image to sentence matching benchmarks. In image search, we show the percent of correct matches for the top retrieval out of all test images (r@1 for search). In image annotation, given a query image, fetching one of five matching sentences is considered a success. Recall rates for the top five (r@5) denote the cases in which a successful match exists in any of the top five results. The experiments reported for regularized CCA, RCCA, and our 2-way net all use the same sentence and image representation. Sentences are represented as the concatenation of the GMM-FV and the HGLMM-FV representations of~\cite{Klein_2015_CVPR}. 
\label{tab:flickrs}. Image is represented with the last dense connected of the CNN used in ~\cite{Klein_2015_CVPR}. }
\label{tab:flickrs}
\end{table*}%

\subsection{Ablation analysis}
We perform an ablation analysis aimed at isolating the effect of the various architectural novelties suggested. Experiments were conducted on the Flickr8k, Flickr30k, MNIST and XRMB datasets. Each experiment uses the baseline configuration used in previous experiments with only one alternation.

\noindent{\bf Batch Normalization}
For this experiment, we used different settings for the BN layer. The configuration settings include: (1) without BN, (2) with conventional BN (before ReLU) without regularizing $\gamma$, (3) with post-ReLU BN, without regularizing $\gamma$, (4) using BN before the ReLU with $\lambda_\gamma = 0.05$, and (5) our proposed method: BN applied only after ReLU with $\lambda_\gamma = 0.05$.  Tab.~\ref{tab:results} report the performance of the various configurations in terms of correlation and the mean variance of all features on the validation set.

\begin{table*}[h]
\centering
    \resizebox{\textwidth}{!}{\begin{tabular}{|l|c|c|c|c|c|c|c|c|c|c|c|c|}
    \hline
   \multirow{2}{*}{Scenario} & \multicolumn{3}{c|}{Flickr8k} & \multicolumn{3}{c|}{Flickr30k}&          \multicolumn{3}{c|}{MNIST} & \multicolumn{3}{c|}{XRMB} \\
   \cline{2-13}
   & Corr & Var x & Var y & Corr & Var x & Var y & Corr & Var x & Var y & Corr & Var x & Var y \\
    \hline
      \makecell[l]{Suggested method}  & \textbf{1758}& 0.65 & 0.64 & \textbf{2135} & 0.41 & 0.43&\textbf{ 49.15} & 1.32 & 1.27& \textbf{110.18} & 1.08 & 1.06\\
    \hline
    No BN  & 1482 & \num{1.90} & \num{1.71} & 1562 & 1.38 & 1.40 & 13.14 & 0 & 0 & 25.58 & 0 & 0\\
    before ReLU, $\lambda_\gamma=0$ & 1313 & 0.66 & 0.44 & 1385 & 0.37 & 0.28 & 48.40 & 0.18 & 0.18& 107.55 & \num{0.15} & \num{0.15}\\
    after ReLU, $\lambda_\gamma=0$ & 1598 & 1.34 & 1.25 & 1655 & 0.73 & 0.74& 48.98 & 0.38 & 0.37& 109.42 & 0.40 & 0.39\\
    \makecell[l]{before ReLU, $\lambda_\gamma > 0$} & 1423 & 0.33 & 0.21 & 1322 & 1.80 & 0.96 & 48.76 & 0.73 & 0.72& 108.79  & 0.50 & 0.50\\
    \hline
    No Dropout & 1091 & 0.34 & 0.33 & 1446 & 0.57 & 0.52& 49.00 & 1.33 & 1.33& 109.69 & 0.79 & 0.79\\
    \makecell[l]{Conventional dropout} & 1557 & 0.17 & 0.17 & 1658 & 0.12 & 0.14& 48.77 & 1.90 & 1.90& 93.24 & 0.24 & 0.16\\
    \hline
  
    \end{tabular}}
   \caption{Ablation study on the Flickr8k, Flickr30k, MNIST and XRMB datasets, testing various batch normalization (BN), variance regularization and dropout options. We measure the variance in both views, $X$ and $Y$ (averaging the variance of all dimensions), and the obtained correlation. The suggested method is to apply BN only after ReLU with $\lambda_\gamma = 0.05$ and to employ tied dropout. All BN variants employ tied dropout with probability of 0.5. All dropout variants apply BN similarly to the suggested method.}
\label{tab:results}%
\end{table*}	

\begin{table}[t]
\centering
    \begin{tabular}{|l|c|c|}
    \hline
    Scenario & Search r@1  & Annotate r@1 \\
    \hline
    Suggested method  & \textbf{29.3} & \textbf{43.4}\\
    \hline
    No BN  & 21.1 & 25.6\\
    before ReLU, $\lambda_\gamma=0$ & 26.9 & 39.6\\
    after ReLU, $\lambda_\gamma=0$ & 27.9 & 40.9\\
    \hline
    No Dropout & 25.64 & 36.6\\
    Conventional dropout & 29.04 & 42.1\\
    \hline
     \end{tabular}
     \caption{Recall results on Flickr8k for the same experiments as described at Tab. ~\ref{tab:results}.}
\label{tab:results_recall}%
\end{table}

As Tab.~\ref{tab:results} shows, batch normalization has a profound effect on the network's results. Results taken without batch normalization were trained with lower learning rate, using higher learning rate prevented the training from converging. We can also see that using the $1/\gamma$ regularization term significantly increases the variance of the hidden representation, which, in turn, stabilizes the training process and improves correlation. The effect studied in Section~\ref{sec:bn} is clearly visible in the ablation study, positioning the BN layer after the Leaky ReLU prevents an unbalance representations as can be seen by the difference in variances, which increases the correlation of two representation significantly. Tab. ~\ref{tab:results_recall} contains r@1 results for the same experiments on the Flickr8k dataset. As in Tab. ~\ref{tab:results} out suggested configuration achieves the base recall rates.   

\noindent{\bf Tied Dropout}
We trained the same base configuration as described above. We tested our proposed method using a conventional dropout instead of a tied dropout and removing dropout altogether.
In all experiments, the dropout probability $p$ was set at 0.5.

As can be seen, the performance drops when using the conventional dropout instead of the proposed tied dropout layer. The benefits of the tied dropout layer are most significant on the large datasets Flickr8k and Flickr30k, where over-fitting is likely. The shrinking effect discussed in Section ~\ref{sec:drop} is clearly visible and is manifested as low variance of the output of the model based on conventional dropout, compared to a much higher variance when using the tied dropout.   %

\noindent{\bf Leaky ReLU} We also tested the contribution of other parameters on the model's performance. One of the major benefits was using leaky ReLU non-linearity. Using conventional ReLU resulted in large correlation loss of about $33\%$ (1192 total correlation) for Flickr8k. 
\noindent{\bf Loss terms} Another aspect we tested is the effect of various loss terms on correlation and recall rates. Removing $L_h$ term results in a $31\%$ (1230) decrease of correlation. This settles with Lemma \ref{lemma:sumcorr} which links the output's correlation and $L_h$ loss term. While the $L_h$ loss increases the output's correlation, the reconstruction loss terms $L_x$ and $L_y$ decreases the result's correlation. Removing them both increases correlation by $56\%$ (2752). While the correlation produced between the two views is higher without the two reconstruction losses, the dimensions of each representation are highly correlated resulting in a decrease of $87\%$ in image search and $91\%$ in image annotation performance as measured by recall@1: from the full method's performance of 29.3 and 43.4 for the tasks of image search and image annotation to 4.0 and 3.9 respectively.
\noindent{\bf Regularization} The effect for $R_\gamma$ can be viewed in Tab. ~\ref{tab:results}. Removing the $R_decov$ results in a decrease of all measures. Image search r@1 and r@5 results decrease by $14\%$ and $10\%$ respectively and the image annotation r@1 and r@5 results decrease by $10\%$ and $8\%$ respectively. Moreover, the correlation is reduced by $4\%$. 
\noindent{\bf Locally dense layer} To test the effect of the proposed locally dense layer, we trained our model on Flickr30k with a regular dense layer of the same size (16000 neurons) and with a regular dense layer of half the size. Image annotation r@1(r@5) results degrade by $7\%$($3\%$) and image search by $1\%$($1\%$) when using conventional 16000 neurons dense layer. Using dense layer half the size results in a drop of $13\%$($9\%$) for image annotation r@1(r@5) rates and $11\%$($8\%$) for image search recall rates r@1(r@5).

\noindent{\bf Parameter sensitivity:} Fig.~\ref{fig:leak}(a) shows the effect of different leakiness coefficient values on the correlation as measured on the validation sets of the MNIST and XRMB data sets. The results were obtained by training the network using leakiness coefficients ranging between 0 and 0.7. As can be seen, there is a large region of values that provide better performance than the conventional zero-leakiness ReLU. Fig.~\ref{fig:leak}(b) shows the effect of the regularization weight $\lambda_\gamma$ that controls the learned variance of the BN layer. The value used in our experiments seems to be beneficial and lies at a relatively wide high-performance plateau.
\begin{figure*}[t]%
    \centering

    \resizebox{\textwidth}{!}{\begin{subfigure}{\columnwidth}
\begin{tikzpicture}[scale=0.82, transform shape]
\begin{axis}[
    xlabel={$\alpha$},
    ylabel={\% Correlation},
    xmin=0, xmax=0.7,
    ymin=97, ymax=99,
    xtick={0,0.1,0.2,0.3,0.4,0.5,0.6,0.7},
    ytick={97, 97.5, 98, 98.5, 99},
    legend pos=south west,
    ymajorgrids=true,
    grid style=dashed,
]
 
\addplot[
    color=blue,
    mark=square,
    ]
    coordinates {
    (0,98.0)(0.1,98.2)(0.2,98.42)(0.3,98.35)(0.4,98.34)(0.5,98.19)(0.6,97.69)(0.7,97.44)

    };
    \addplot[color=red, mark=triangle,]
    coordinates {
    (0,98.04)(0.1,98.24)(0.2,98.38)(0.3,98.34)(0.4,98.18)(0.5,98.48)(0.6,98.28)(0.7,98.18)
    };
    \legend{XRMB, MNIST}
 
\end{axis}
\end{tikzpicture}

\caption{}
\end{subfigure}

\begin{subfigure}{\columnwidth}
    \begin{tikzpicture}[scale=0.82, transform shape]
\begin{axis}[
    xlabel={$\lambda$},
    ylabel={\% Correlation},
    xmin=0, xmax=1,
    ymin=96, ymax=99,
    xtick={0, 0.1, 0.2, 0.3, 0.4, 0.5, 0.6, 0.7, 0.8, 0.9, 1
},
    ytick={96, 96.5, 97, 97.5, 98, 98.5, 99},
    legend pos=south west,
    ymajorgrids=true,
    grid style=dashed,
]
 
\addplot[
    color=blue,
    mark=square,
    ]
    coordinates {
(0, 97.71)
(0.05, 98.32)
(0.1, 98.47)
(0.15, 97.28)
(0.2, 97.97)
(0.25, 97.79)
(0.3, 97.72)
(0.35, 98.09)
(0.4, 97.79)
(0.45, 97.93)
(0.5, 97.63)
(0.55, 98.01)
(0.6, 97.34)
(0.65, 96.21)
(0.7, 97.22)
(0.75, 97.56)
(0.8, 96.91)
(0.9, 97.19)
(0.95, 96.13)
(1, 97.84)
    };
    \addplot[color=red, mark=triangle,]
    coordinates {
(0, 96.04)
(0.05, 98.48)
(0.1, 98.50)
(0.15, 98.44)
(0.2, 98.40)
(0.25, 98.54)
(0.3, 98.36)
(0.35, 98.42)
(0.4, 98.50)
(0.45, 97.68)
(0.5, 98.08)
(0.55, 97.04)
(0.6, 97.82)
(0.65, 98.44)
(0.7, 96.98)
(0.75, 98.14)
(0.8, 98.10)
(0.9, 96.70)
(0.95, 98.06)
(1, 97.70)    
};
    \legend{XRMB, MNIST}
    \end{axis}
\end{tikzpicture}
\caption{}
\end{subfigure}}

    \caption{(a) The effect of the leakiness parameter on the MNIST and XRMB benchmarks, as measured on the validation set using the sum of correlations divided by the dimension (in percent). The solid red line depicts the MNIST results; the dashed black line depicts the XRMB results. (b) A similar plot showing the effect of coefficient $\lambda_\gamma$.}
    \label{fig:leak}%
\end{figure*}
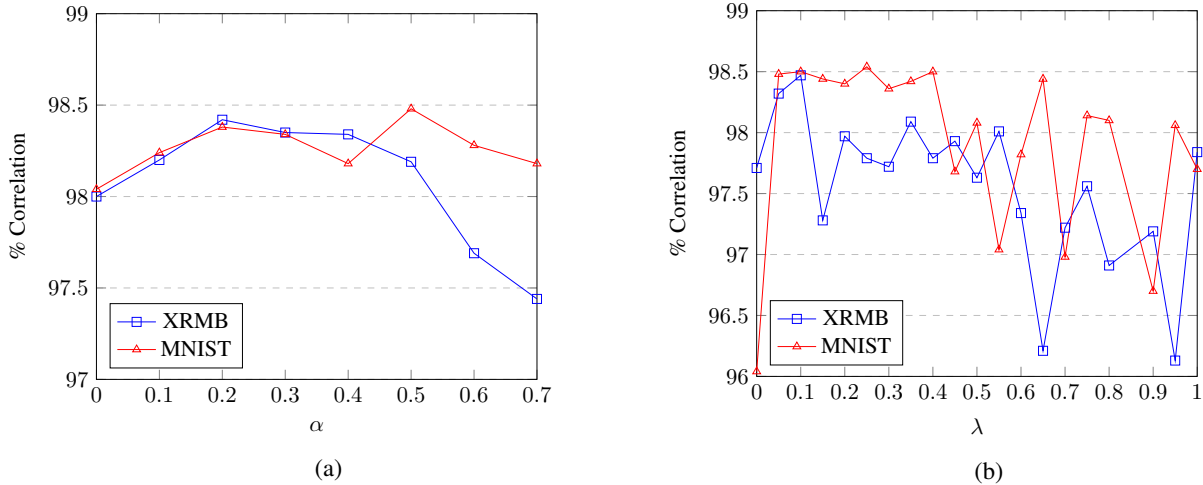

\section{Conclusions}

In this paper, we present a method for linking paired samples from two sources. The method significantly outperforms all literature methods in the highly applicable and well studied domain of correlation analysis, including the classical methods, their modern variants, and the recent deep correlation methods. We are unique in that we employ a tied 2-way architecture, reconstructing , and unlike most methods, we employ the Euclidean loss. In order to promote an effective training, we introduce a series of contributions that are aimed at maintaining the variance of the learned representations. Each of these modifications is  provided with an analysis that explains its role and together they work hand in hand in order to provide the complete architecture, which is highly accurate. 

Our method is generic and can be employed in any computer vision domain in which two data modalities are used. In addition, our contributions could also help in training univariate regression problems. In the literature, the Euclidean loss is often combined with other losses~\cite{NIPS2015_5638,Zhang_2015_CVPR}, or replaced by an alternative loss~\cite{Levy_2015_ICCV} in order to mitigate the challenges of training regression problems. Our variance injection method can be easily incorporated into any existing network.

As future work, we would like to continue exploring the use of tied 2-Way networks for matching views from different domains. In almost all of our trained networks, the biases of the batch normalization layers in the solutions tend to have very low values. These biases can probably be eliminated altogether. In addition, in many encoder/decoder schemes, layers are added gradually during training. It is possible to adopt such a scheme to our framework, adding hidden layers in the middle of the network one by one.

\bibliographystyle{plain}
\bibliography{deepl2match}
\end{document}